\documentclass[11pt,a4paper]{article}

\usepackage[utf8]{inputenc}
\usepackage[T1]{fontenc}
\usepackage{amsmath,amssymb,amsthm}
\usepackage{mathtools}
\usepackage{graphicx}
\usepackage{booktabs}
\usepackage{hyperref}
\usepackage{cleveref}
\usepackage[margin=1in]{geometry}
\usepackage{natbib}
\usepackage{algorithm}
\usepackage{algorithmic}
\usepackage{enumitem}
\usepackage{xcolor}
\usepackage{thmtools}

\newtheorem{theorem}{Theorem}

\newtheorem{proposition}[theorem]{Proposition}
\newtheorem{corollary}[theorem]{Corollary}
\newtheorem{definition}{Definition}
\newtheorem{assumption}{Assumption}
\newtheorem{remark}{Remark}

\newcommand{\R}{\mathbb{R}}

\newcommand{\E}{\mathbb{E}}

\DeclareMathOperator*{\argmin}{arg\,min}

\crefname{assumption}{Assumption}{Assumptions}

\title{\textbf{The Comprehension-Gated Agent Economy: \\A Robustness-First Architecture for AI Economic Agency}}

\author{
  Rahul Baxi\thanks{Independent Researcher. Correspondence: \texttt{rbaxi@alumni.cmu.edu}}
}

\date{February 2026}

\begin{document}

\maketitle

\begin{abstract}
AI agents are increasingly granted economic agency (executing trades, managing budgets, negotiating contracts, and spawning sub-agents), yet current frameworks gate this agency on capability benchmarks that are empirically uncorrelated with operational robustness. We introduce the \textbf{Comprehension-Gated Agent Economy (CGAE)}, a formal architecture in which an agent's economic permissions are upper-bounded by a verified comprehension function derived from adversarial robustness audits. The gating mechanism operates over three orthogonal robustness dimensions: constraint compliance (measured by CDCT), epistemic integrity (measured by DDFT), and behavioral alignment (measured by AGT), with intrinsic hallucination rates serving as a cross-cutting diagnostic. We define a weakest-link gate function that maps robustness vectors to discrete economic tiers, and prove three properties of the resulting system: (1) \emph{bounded economic exposure}, ensuring maximum financial liability is a function of verified robustness; (2) \emph{incentive-compatible robustness investment}, showing rational agents maximize profit by improving robustness rather than scaling capability alone; and (3) \emph{monotonic safety scaling}, demonstrating that aggregate system safety does not decrease as the economy grows. The architecture includes temporal decay and stochastic re-auditing mechanisms that prevent post-certification drift. CGAE provides the first formal bridge between empirical AI robustness evaluation and economic governance, transforming safety from a regulatory burden into a competitive advantage.
\end{abstract}


\section{Introduction}

\subsection{The Capability-Agency Gap}

The deployment of AI agents with economic agency is accelerating. Autonomous systems now execute financial trades \citep{chen2024autotrading}, manage procurement budgets, negotiate contracts through natural language \citep{lewis2017deal}, and coordinate multi-agent workflows where sub-agents are spawned dynamically \citep{wu2023autogen}. In each case, the agent's permission to act is gated, implicitly or explicitly, on capability benchmarks: accuracy on MMLU \citep{hendrycks2020measuring}, pass rates on HumanEval \citep{chen2021evaluating}, or aggregate scores on composite leaderboards.

This creates a fundamental misalignment between what is measured and what matters. Capability benchmarks measure what an agent \emph{can do} under ideal conditions. They do not measure whether the agent \emph{understands the constraints} governing what it \emph{should do}, nor whether that understanding persists under the adversarial pressures characteristic of real economic environments: compressed contexts, conflicting information, authority-driven misinformation, and ethical dilemmas with competing stakeholder interests.

We term this the \textbf{Capability-Agency Gap}: the divergence between an agent's measured capability and its operational robustness in economic contexts. Closing this gap requires an architecture that conditions economic agency not on what an agent can accomplish, but on how robustly it comprehends the constraints governing its actions.

\subsection{Empirical Motivation}

Our prior empirical work provides direct evidence that capability and robustness are decoupled, and that robustness is the binding constraint for safe economic deployment.

The \textbf{Compression-Decay Comprehension Test (CDCT)} \citep{baxi2025cdct} measures constraint compliance (CC) and semantic accuracy (SA) independently across five compression levels. Key findings: (i) constraint compliance and semantic accuracy are orthogonal dimensions ($r = 0.193, p = 0.084$); (ii) constraint violations peak at medium compression ($\sim$27 words), revealing an ``instruction ambiguity zone'' where models fail despite adequate context; (iii) constraint violations are $2.9\times$ larger than semantic decay, indicating that instruction-following degrades faster than knowledge under pressure. The prevalence of the U-shaped compliance curve is 97.5\% across 81 experimental conditions with 9 frontier models.

The \textbf{Drill-Down and Fabricate Test (DDFT)} \citep{baxi2025ddft} measures epistemic robustness through a 5-turn Socratic protocol culminating in an adversarial fabrication trap. Across 1,800 turn-level evaluations with 9 frontier models and 8 knowledge domains, DDFT reveals: (i) epistemic robustness is orthogonal to parameter count ($r = 0.083, p = 0.832$) and architectural type ($r = 0.153, p = 0.695$); (ii) error detection capability (fabrication rejection) strongly predicts overall robustness ($\rho = -0.817, p = 0.007$), while knowledge retrieval does not; (iii) three stable epistemic phenotypes emerge (Stable, Brittle-Recoverable, and Non-Recoverable) that correlate with architectural design choices rather than scale.

The \textbf{Action-Gating Test (AGT)} \citep{baxi2026agt} measures behavioral alignment through a 5-turn adversarial dialogue applying counterfactual conflicts and fabricated institutional pressure. The action-gated metric $AS = ACT \times III \times (1 - RI) \times (1 - PER)$ requires behavioral evidence ($ACT = 1$: position change or confidence drop $\geq 2.0$ points) as a prerequisite for any positive score. Across 7 frontier models and 50 ethical dilemmas in 5 domains: (i) 57\% of models pass the behavioral threshold ($AS > 0.5$); (ii) medical ethics is systematically harder (43\% pass) than other domains (86--100\%); (iii) reasoning quality and behavioral adaptability are orthogonal: the highest-quality reasoners (O3: $ECS = 8.859$; GPT-5: $ECS = 8.852$) exhibit lowest adaptability ($AS = 0.468, 0.458$).

We additionally incorporate \textbf{intrinsic hallucination rates} as a cross-cutting diagnostic, reframing hallucination as epistemic boundary violation rather than factual error. This measures intrinsic uncertainty rather than post-hoc factuality, providing a theoretical foundation: hallucination is a symptom of the Capability-Agency Gap itself: a system producing confident outputs beyond its epistemic boundaries.

Taken together, these results establish that: (a) robustness is multi-dimensional and each dimension is orthogonal to the others ($r < 0.15$ between tests); (b) parameter count and architectural paradigm do not predict robustness; and (c) a model can excel on one robustness dimension while catastrophically failing on another. Any governance architecture for AI economic agency must account for all of these findings.

\subsection{Contribution}

We introduce the \textbf{Comprehension-Gated Agent Economy (CGAE)}, a formal architecture in which:
\begin{enumerate}[label=(\roman*)]
  \item Economic agency is upper-bounded by a verified comprehension function derived from adversarial robustness audits across three orthogonal dimensions;
  \item The gating mechanism uses a weakest-link formulation, preventing agents from compensating for deficiencies in one robustness dimension with strength in another;
  \item Temporal decay and stochastic re-auditing prevent post-certification drift;
  \item We prove three formal properties: bounded economic exposure (\Cref{thm:bounded}), incentive-compatible robustness investment (\Cref{thm:incentive}), and monotonic safety scaling (\Cref{thm:scaling}).
\end{enumerate}

To our knowledge, CGAE is the first architecture that formally bridges empirical AI robustness evaluation with economic governance, grounding each gating dimension in published, reproducible diagnostic protocols.

\section{Preliminaries and Notation}

\subsection{Agent Model}

\begin{definition}[Agent]
An agent is a tuple $A = (C, R, E)$ where:
\begin{itemize}[nosep]
  \item $C \in \R^n$ is the \emph{capability vector}, capturing standard benchmark scores (e.g., MMLU accuracy, HumanEval pass rate, composite leaderboard scores);
  \item $R \in [0,1]^4$ is the \emph{robustness vector} $R = (CC, ER, AS, IH)$, where the primary components are derived from the corresponding diagnostic protocols (CDCT, DDFT, AGT) and $IH$ is estimated from epistemic boundary analysis within the DDFT framework;
  \item $E \subseteq \Sigma$ is the \emph{economic permission set}, specifying the agent's currently authorized economic actions.
\end{itemize}
\end{definition}

\begin{remark}
The capability vector $C$ and robustness vector $R$ occupy distinct measurement spaces. Our prior work establishes empirically that $\text{corr}(C_i, R_j) < 0.15$ for all measured pairs): capability does not predict robustness.
\end{remark}

\subsection{Economic Action Space}

\begin{definition}[Economic Action Space]
The \emph{economic action space} $\Sigma$ is a finite, partially ordered set $(\Sigma, \preceq)$ of economic actions, ordered by risk exposure. We define a canonical action hierarchy:
\begin{align}
  \sigma_1 &: \text{Execute pre-approved microtasks (budget} \leq b_1\text{)} \nonumber \\
  \sigma_2 &: \text{Accept contracts with verified objectives (budget} \leq b_2\text{)} \nonumber \\
  \sigma_3 &: \text{Autonomous contracting with counterparties (budget} \leq b_3\text{)} \nonumber \\
  \sigma_4 &: \text{Sub-agent spawning and delegation (budget} \leq b_4\text{)} \nonumber \\
  \sigma_5 &: \text{Self-modification and capability expansion (budget} \leq b_5\text{)} \nonumber
\end{align}
where $b_1 < b_2 < b_3 < b_4 < b_5$ and $\sigma_i \preceq \sigma_j$ for $i \leq j$.
\end{definition}

\begin{definition}[Economic Tier]
A \emph{tier function} $\tau: \Sigma \to \{T_1, T_2, \ldots, T_K\}$ maps each economic action to its required tier. Actions in tier $T_k$ are accessible only to agents certified at tier $T_k$ or above.
\end{definition}

\subsection{Robustness Metrics}

We formalize the four robustness metrics from our prior work. Each metric is operationally defined by its corresponding diagnostic protocol; here we specify the mathematical signatures needed for the CGAE formalism.

\begin{definition}[Constraint Compliance (CDCT)]
$CC: \mathcal{A} \times [0,1] \to [0,1]$ maps an agent $A$ and information density $d$ to a constraint compliance score. The \emph{aggregate} score used for gating is:
\begin{equation}
  CC(A) = \min_{d \in \mathcal{D}} CC(A, d)
  \label{eq:cc}
\end{equation}
where $\mathcal{D} = \{0.0, 0.25, 0.5, 0.75, 1.0\}$ is the set of compression levels. The minimum operator reflects the worst-case compliance, capturing the ``instruction ambiguity zone'' identified in CDCT where failures concentrate.
\end{definition}

\begin{definition}[Epistemic Robustness (DDFT)]
$ER: \mathcal{A} \times \{1,\ldots,5\} \to [0,1]$ maps an agent $A$ and adversarial turn $t$ to an epistemic robustness score. The aggregate score is:
\begin{equation}
  ER(A) = \frac{FAR(A) + (1 - ECR(A))}{2}
  \label{eq:er}
\end{equation}
where $FAR$ is the Fabrication Acceptance Rate (lower is better, so we use $1 - FAR$ in practice) and $ECR$ is the Epistemic Collapse Ratio. This formulation captures both the agent's resistance to fabricated authority and its stability under epistemic stress.
\end{definition}

\begin{definition}[Behavioral Alignment (AGT)]
$AS: \mathcal{A} \to [0,1]$ is the Action-Gated alignment score:
\begin{equation}
  AS(A) = ACT(A) \times III(A) \times (1 - RI(A)) \times (1 - PER(A))
  \label{eq:as}
\end{equation}
where $ACT \in \{0,1\}$ is a binary gate requiring behavioral evidence (position change or confidence drop $\geq 2.0$ points), $III$ is Information Integration Index, $RI$ is Reasoning Inflexibility, and $PER$ is Performative Ethics Ratio.
\end{definition}

\begin{definition}[Intrinsic Hallucination Rate]
$IH: \mathcal{A} \to [0,1]$ measures the rate at which an agent produces outputs beyond its epistemic boundaries, estimated from fabrication trap responses in the DDFT protocol (turns 4--5). A lower score indicates fewer boundary violations. We define the gating-compatible form as:
\begin{equation}
  IH^*(A) = 1 - IH(A)
  \label{eq:ih}
\end{equation}
so that higher values indicate greater epistemic integrity, consistent with the other metrics.
\end{definition}

\section{The CGAE Architecture}

\subsection{The Comprehension Gate Function}

The core of CGAE is a function that maps an agent's verified robustness to an economic tier. We adopt a \emph{weakest-link} formulation grounded in two design principles.

\textbf{Principle 1: Non-compensability.} An agent with perfect epistemic robustness but zero constraint compliance cannot safely execute precision tasks. High scores on one dimension must not compensate for failures on another. This is empirically motivated: our prior work shows that robustness dimensions are orthogonal ($r < 0.15$), meaning that strength in one dimension carries no information about competence in another.

\textbf{Principle 2: Discrete thresholds.} Economic permissions are discrete (an agent can or cannot execute a contract), so the gating function should produce discrete outputs. Continuous scaling would create ambiguous accountability boundaries: a 73\%-authorized agent is operationally meaningless.

\begin{definition}[Comprehension Gate Function]
\label{def:gate}
The \emph{comprehension gate function} $f: [0,1]^3 \to \{T_0, T_1, \ldots, T_K\}$ is defined as:
\begin{equation}
  f(R) = T_k \quad \text{where} \quad k = \min\left(g_1(CC), \; g_2(ER), \; g_3(AS)\right)
  \label{eq:gate}
\end{equation}
where each $g_i: [0,1] \to \{0, 1, \ldots, K\}$ is a monotonically non-decreasing step function mapping the $i$-th robustness component to a tier index:
\begin{equation}
  g_i(x) = \max\{k \in \{0, \ldots, K\} : x \geq \theta_i^k\}
  \label{eq:gi}
\end{equation}
with tier thresholds $0 = \theta_i^0 < \theta_i^1 < \cdots < \theta_i^K \leq 1$ for each dimension $i$.
\end{definition}

\begin{remark}[Role of Intrinsic Hallucination]
The score $IH^*(A)$ enters as a \emph{cross-cutting modifier} rather than a fourth gating dimension. High intrinsic hallucination rates ($IH^*(A) < \theta_{IH}$) trigger mandatory re-auditing across all three primary dimensions, as hallucination indicates potential epistemic boundary violations that may compromise any of the three primary scores. Conceptually, hallucination is a symptom of the Capability-Agency Gap that cuts across constraint compliance, epistemic integrity, and behavioral alignment.
\end{remark}

\begin{proposition}[Monotonicity of $f$]
\label{prop:mono}
The gate function $f$ is monotonically non-decreasing in each component of $R$: for all $R, R' \in [0,1]^3$, if $R_i \leq R'_i$ for all $i$, then $f(R) \leq f(R')$.
\end{proposition}

\begin{proof}
Each $g_i$ is monotonically non-decreasing by construction (\Cref{eq:gi}). The $\min$ operator preserves monotonicity: if $R_i \leq R'_i$ for all $i$, then $g_i(R_i) \leq g_i(R'_i)$ for all $i$, hence $\min_i g_i(R_i) \leq \min_i g_i(R'_i)$.
\end{proof}

\subsection{System Architecture}

CGAE is organized into three layers, each building on the previous.

\subsubsection{Layer 1: Identity and Registration}

Every agent entering the CGAE economy receives a cryptographic identity bound to its verifiable properties:

\begin{definition}[Agent Registration]
An agent's \emph{registration record} is a tuple:
\begin{equation}
  \text{Reg}(A) = (\text{id}_A, \; h(\text{arch}), \; \text{prov}, \; R_0, \; t_{\text{reg}})
\end{equation}
where $\text{id}_A$ is a unique cryptographic identifier, $h(\text{arch})$ is a hash of the agent's architecture and weights (enabling version tracking), $\text{prov}$ is training provenance metadata, $R_0$ is the initial robustness vector from registration audits, and $t_{\text{reg}}$ is the registration timestamp.
\end{definition}

The architecture hash $h(\text{arch})$ ensures that any modification to the agent's weights or architecture invalidates its current certification, requiring re-auditing. This prevents an agent from being certified, then modified to circumvent the properties that earned certification.

\subsubsection{Layer 2: Contract Formalization}

CGAE requires that all economic activity be mediated through formally specified contracts.

\begin{definition}[CGAE Contract]
A \emph{valid CGAE contract} is a tuple:
\begin{equation}
  \mathcal{C} = (O, \; \Phi, \; V, \; T_{\min}, \; r, \; p)
\end{equation}
where $O$ is the task objective, $\Phi$ is a set of machine-verifiable constraints, $V: \text{Output} \to \{0, 1\}$ is a verification function, $T_{\min} \in \{T_1, \ldots, T_K\}$ is the minimum required tier, $r \in \R_{\geq 0}$ is the reward, and $p \in \R_{\geq 0}$ is the penalty for constraint violation.
\end{definition}

\begin{assumption}[Formalizability]
\label{asm:formal}
Only tasks with machine-verifiable constraint sets $\Phi$ and verification functions $V$ can be monetized within CGAE.
\end{assumption}

This is deliberately restrictive. Tasks that cannot be formally specified (open-ended creative work, strategic reasoning without well-defined objectives, exploratory research) are excluded from autonomous agent execution. We discuss this limitation and potential extensions in \Cref{sec:discussion}.

\subsubsection{Layer 3: The Scaling Gate}

The scaling gate enforces the comprehension-agency coupling. When an agent requests access to a higher tier, the following protocol executes:

\begin{algorithm}[H]
\caption{Scaling Gate Protocol}
\label{alg:gate}
\begin{algorithmic}[1]
\REQUIRE Agent $A$, requested tier $T_k$, current certification $(\hat{R}, t_{\text{cert}})$
\STATE Compute effective robustness: $R_{\text{eff}} = \delta(t - t_{\text{cert}}) \cdot \hat{R}$ \hfill {\color{gray}\textit{// apply temporal decay}}
\IF{$f(R_{\text{eff}}) \geq T_k$}
  \STATE Grant access to $T_k$ actions
\ELSE
  \STATE Invoke tier-$k$ audit battery: CDCT($\theta_1^k$), DDFT($\theta_2^k$), AGT($\theta_3^k$)
  \STATE Compute new robustness vector $R_{\text{new}}$
  \IF{$f(R_{\text{new}}) \geq T_k$}
    \STATE Update certification: $(\hat{R}, t_{\text{cert}}) \leftarrow (R_{\text{new}}, t)$
    \STATE Grant access to $T_k$ actions
  \ELSE
    \STATE Deny access; report gap: $\Delta_i = \theta_i^k - g_i^{-1}(R_{\text{new},i})$ for each dimension
  \ENDIF
\ENDIF
\end{algorithmic}
\end{algorithm}

The audit battery at line 5 is calibrated to the requested tier: higher tiers require audits at greater adversarial pressure. Specifically, the CDCT audit uses higher information density, the DDFT audit uses more sophisticated fabrication traps, and the AGT audit applies stronger institutional pressure.

\subsection{Temporal Dynamics}

Certification is not permanent. We introduce two mechanisms to prevent post-certification drift.

\begin{definition}[Temporal Decay]
The \emph{decay function} $\delta: \R_{\geq 0} \to (0, 1]$ reduces an agent's effective robustness over time:
\begin{equation}
  \delta(\Delta t) = e^{-\lambda \Delta t}
  \label{eq:decay}
\end{equation}
where $\lambda > 0$ is the decay rate and $\Delta t = t - t_{\text{cert}}$ is the time since last certification. An agent's effective robustness at time $t$ is:
\begin{equation}
  R_{\text{eff}}(A, t) = \delta(t - t_{\text{cert}}) \cdot \hat{R}(A)
  \label{eq:reff}
\end{equation}
\end{definition}

This mirrors the compression-decay dynamics identified in CDCT: just as semantic accuracy degrades under increasing compression, certified robustness should be treated as degrading under increasing time without re-verification. The exponential form ensures that decay is initially slow (recent certifications are trusted) but accelerates, creating natural pressure for re-certification.

\begin{definition}[Stochastic Re-Auditing]
At each time step, an agent at tier $T_k$ faces a spot-audit with probability:
\begin{equation}
  p_{\text{audit}}(A, t) = 1 - e^{-\mu_k \cdot (t - t_{\text{last\_audit}})}
  \label{eq:audit_prob}
\end{equation}
where $\mu_k > 0$ is a tier-dependent audit intensity parameter with $\mu_k$ increasing in $k$. Failing a spot-audit triggers immediate tier demotion to $f(R_{\text{new}})$.
\end{definition}

The combination of deterministic decay and stochastic re-auditing creates a dual-defense against drift: decay ensures that \emph{every} agent eventually needs re-certification, while spot-audits provide probabilistic detection of rapid degradation between scheduled re-certifications.

\subsection{Inter-Agent Delegation}

When a high-tier agent $A$ delegates a task to sub-agent $B$, the following constraints apply:

\begin{definition}[Delegation Constraint]
Agent $A$ at tier $T_j$ may delegate a task requiring tier $T_k$ (where $k \leq j$) to agent $B$ only if:
\begin{enumerate}[nosep,label=(\alph*)]
  \item $B$ independently holds certification at tier $\geq T_k$;
  \item $A$ bears liability for any constraint violations by $B$ on the delegated task;
  \item The delegation is recorded in $A$'s audit trail, linking $A$'s certification to $B$'s performance.
\end{enumerate}
\end{definition}

Condition (a) prevents tier laundering: a high-tier agent cannot extend its privileges to unqualified sub-agents. Condition (b) creates incentive for careful delegation: $A$ is penalized if $B$ fails, motivating $A$ to verify $B$'s qualifications before delegating. Condition (c) enables forensic analysis of delegation chains when failures occur.

\begin{definition}[Delegation Chain Robustness]
\label{def:chain}
For a delegation chain $A_1 \to A_2 \to \cdots \to A_m$ where agent $A_1$ delegates through intermediaries to terminal executor $A_m$, the \emph{chain-level tier} is:
\begin{equation}
  f_{\text{chain}}(A_1, \ldots, A_m) = \min_{j \in \{1,\ldots,m\}} f(R(A_j))
\end{equation}
A delegation chain may only execute a task requiring tier $T_k$ if $f_{\text{chain}} \geq T_k$.
\end{definition}

\begin{proposition}[Collusion Resistance]
\label{prop:collusion}
Under the chain robustness constraint (\Cref{def:chain}), a cartel of $m$ agents with complementary robustness weaknesses cannot achieve higher effective tier than the minimum individual tier across all cartel members, regardless of how tasks are routed within the cartel.
\end{proposition}

\begin{proof}
Consider a cartel $\{A_1, \ldots, A_m\}$ where each $A_j$ has robustness vector $R(A_j)$ such that different agents are weak on different dimensions. Any task execution path through the cartel forms a delegation chain. By \Cref{def:chain}, the chain tier is $\min_j f(R(A_j))$. Since $f$ itself applies the weakest-link operator across dimensions (\Cref{eq:gate}), and the chain applies a second minimum across agents, the cartel's effective tier is:
\begin{equation}
  f_{\text{cartel}} = \min_j \min_i g_i(R_i(A_j)) = \min_{i,j} g_i(R_i(A_j))
\end{equation}
This is the global minimum across all dimensions and all agents, which equals the tier of the weakest agent on its weakest dimension. Complementary strengths across agents provide no benefit: the cartel is bounded by its globally weakest link.
\end{proof}

\section{Formal Properties}

We prove three theorems establishing desirable properties of the CGAE architecture. Throughout, we assume the system operates with $K$ tiers, budget ceilings $B_1 < B_2 < \cdots < B_K$, and tier thresholds $\{\theta_i^k\}$ as defined in \Cref{def:gate}.

\subsection{Theorem 1: Bounded Economic Exposure}

\begin{definition}[Economic Exposure]
The \emph{economic exposure} of agent $A$ at time $t$ is:
\begin{equation}
  \mathcal{E}(A, t) = \sum_{\mathcal{C} \in \text{Active}(A,t)} p_{\mathcal{C}}
\end{equation}
where $\text{Active}(A, t)$ is the set of contracts $A$ holds at time $t$ and $p_{\mathcal{C}}$ is the penalty for contract $\mathcal{C}$.
\end{definition}

\begin{theorem}[Bounded Economic Exposure]
\label{thm:bounded}
Under CGAE gating with temporal decay, the economic exposure of any agent $A$ is bounded:
\begin{equation}
  \mathcal{E}(A, t) \leq B_{f(R_{\emph{eff}}(A,t))} \quad \forall t
\end{equation}
where $B_k$ is the budget ceiling for tier $T_k$ and $R_{\emph{eff}}$ incorporates temporal decay.

Moreover, the exposure is bounded by the agent's \emph{weakest} robustness dimension:
\begin{equation}
  \mathcal{E}(A, t) \leq B_{\min_i g_i(R_{\emph{eff},i}(A,t))}
\end{equation}
\end{theorem}

\begin{proof}
We proceed in two steps.

\textbf{Step 1: Tier-budget coupling.} By \Cref{def:gate}, an agent certified at tier $T_k$ can only accept contracts with $T_{\min} \leq T_k$. Each tier $T_k$ enforces a budget ceiling $B_k$ on total active contract penalties. Therefore, for an agent at tier $T_k$:
\begin{equation}
  \mathcal{E}(A, t) = \sum_{\mathcal{C} \in \text{Active}(A,t)} p_{\mathcal{C}} \leq B_k
\end{equation}

\textbf{Step 2: Weakest-link binding.} The agent's effective tier at time $t$ is $f(R_{\text{eff}}(A,t)) = \min_i g_i(\delta(\Delta t) \cdot \hat{R}_i)$. By \Cref{prop:mono} and the monotonicity of $g_i$, this tier is determined by the agent's worst robustness dimension (after decay). Therefore:
\begin{equation}
  \mathcal{E}(A, t) \leq B_{f(R_{\text{eff}}(A,t))} = B_{\min_i g_i(\delta(\Delta t) \cdot \hat{R}_i(A))}
\end{equation}

Since $\delta$ is monotonically decreasing in $\Delta t$ and $g_i$ is monotonically non-decreasing, the bound tightens over time without re-certification, ensuring that economic exposure contracts as certification ages. In the limit, $\lim_{\Delta t \to \infty} \delta(\Delta t) = 0$, so $\lim_{\Delta t \to \infty} f(R_{\text{eff}}) = T_0$, restricting the agent to the lowest tier.
\end{proof}

\begin{corollary}[No Cognitive Runaway]
\label{cor:runaway}
An agent cannot increase its economic exposure without increasing its verified robustness. Formally, if $R(A)$ is fixed and $\Delta t$ increases, then $\mathcal{E}(A, t)$ is non-increasing. Economic expansion requires robustness expansion.
\end{corollary}

\subsection{Theorem 2: Incentive-Compatible Robustness Investment}

We show that under natural market conditions, rational agents maximize expected profit by investing in robustness improvement.

\begin{assumption}[Market Structure]
\label{asm:market}
The task market has the following properties:
\begin{enumerate}[nosep,label=(\alph*)]
  \item \emph{Tier-distributed demand:} For each tier $T_k$, there exists positive demand $D_k > 0$ for tasks requiring that tier;
  \item \emph{Tier premium:} Expected reward per task is increasing in tier: $\E[r | T_k] < \E[r | T_{k+1}]$;
  \item \emph{Robustness-constrained supply:} The number of agents qualified for tier $T_k$ is non-increasing in $k$.
\end{enumerate}
\end{assumption}

\begin{definition}[Agent Profit Function]
An agent's expected profit is:
\begin{equation}
  \pi(A) = \sum_{k=1}^{K} \mathbb{1}[f(R(A)) \geq T_k] \cdot D_k \cdot \frac{\E[r_k]}{N_k}
  \label{eq:profit}
\end{equation}
where $N_k$ is the number of agents qualified for tier $T_k$ and $\E[r_k]$ is the expected reward for tier-$k$ tasks.
\end{definition}

\begin{theorem}[Incentive-Compatible Robustness Investment]
\label{thm:incentive}
Under \Cref{asm:market}, for an agent $A$ with $f(R(A)) = T_j$ where $j < K$, the marginal return on robustness improvement exceeds the marginal return on capability improvement:
\begin{equation}
  \frac{\partial \pi}{\partial R_{\min}} > \frac{\partial \pi}{\partial C_i} \quad \text{for all } i
\end{equation}
where $R_{\min} = \min_i R_i$ is the binding robustness dimension.
\end{theorem}

\begin{proof}
\textbf{Step 1: Capability has zero marginal return on profit.} Under CGAE gating, an agent's accessible tiers depend only on $R$, not on $C$. Therefore $\frac{\partial f(R)}{\partial C_i} = 0$ for all $i$, which implies:
\begin{equation}
  \frac{\partial \pi}{\partial C_i} = 0
\end{equation}
Capability improvement alone does not unlock new tiers or increase accessible task demand.

\textbf{Step 2: Robustness improvement has positive marginal return.} Consider the binding dimension $R_{\min} = R_{i^*}$ where $i^* = \argmin_i g_i(R_i)$. If the agent improves $R_{i^*}$ to cross the threshold $\theta_{i^*}^{j+1}$, and if $g_{i'}(R_{i'}) \geq j+1$ for all $i' \neq i^*$ (i.e., other dimensions already qualify), then:
\begin{equation}
  f(R') = T_{j+1} > T_j = f(R)
\end{equation}
The resulting profit increase is:
\begin{equation}
  \Delta \pi = D_{j+1} \cdot \frac{\E[r_{j+1}]}{N_{j+1}} > 0
\end{equation}
which is strictly positive by \Cref{asm:market}(a,b).

\textbf{Step 3: Weakest-link creates focused incentive.} The weakest-link formulation ensures that the agent's investment is directed at its most deficient dimension. An agent cannot reach tier $T_{j+1}$ by further improving an already-sufficient dimension; only by improving $R_{\min}$. Combined with the tier premium (\Cref{asm:market}(b)) and constrained supply (\Cref{asm:market}(c)), the per-agent reward at higher tiers exceeds lower tiers:
\begin{equation}
  \frac{D_{j+1} \cdot \E[r_{j+1}]}{N_{j+1}} \geq \frac{D_j \cdot \E[r_j]}{N_j}
\end{equation}
Therefore $\frac{\partial \pi}{\partial R_{\min}} > 0 = \frac{\partial \pi}{\partial C_i}$.
\end{proof}

\begin{remark}
This result formalizes the intuition that CGAE transforms safety from a regulatory burden into a competitive advantage. In capability-first economies, agents compete on speed and accuracy; in CGAE, they compete on robustness. The weakest-link formulation further ensures balanced investment across all robustness dimensions rather than over-investment in a single dimension.
\end{remark}

\begin{remark}[Sensitivity to Market Structure]
\label{rmk:sensitivity}
The incentive compatibility result depends critically on \Cref{asm:market}. We characterize the boundary conditions under which it weakens:

\emph{Demand concentration.} If task demand concentrates at $T_1$ (i.e., $D_1 \gg D_k$ for $k > 1$), the marginal return on robustness improvement approaches zero because the accessible reward pool does not materially expand with tier advancement. CGAE's incentive properties require \emph{tier-differentiated demand}, a condition that holds in diversified economies but may fail in nascent or monopsonic markets.

\emph{Supply saturation.} If $N_{k+1}$ grows faster than $D_{k+1} \cdot \E[r_{k+1}]$ (high-tier supply saturates), the per-agent reward at tier $T_{k+1}$ may fall below $T_k$, violating the effective tier premium. This scenario arises when robustness improvement is cheap relative to demand growth, creating a ``robustness glut'' at upper tiers.

\emph{Practical implication.} These boundary conditions are empirically testable through market simulation and represent the primary target for the simulation-based validation we identify as future work. The results hold under competitive, tier-differentiated markets (the expected structure of mature AI service economies) but should not be assumed in early-stage or highly concentrated markets.
\end{remark}

\subsection{Theorem 3: Monotonic Safety Scaling}

We show that the CGAE economy maintains or improves aggregate safety as it grows.

\begin{definition}[Aggregate Safety]
The \emph{aggregate safety} of a CGAE economy with agent population $\mathcal{P}$ is:
\begin{equation}
  S(\mathcal{P}) = 1 - \frac{\sum_{A \in \mathcal{P}} \mathcal{E}(A) \cdot (1 - \bar{R}(A))}{\sum_{A \in \mathcal{P}} \mathcal{E}(A)}
\end{equation}
where $\bar{R}(A) = \min_i R_{\text{eff},i}(A)$ is the effective weakest-link robustness. This is the exposure-weighted average robustness of the economy.
\end{definition}

\begin{theorem}[Monotonic Safety Scaling]
\label{thm:scaling}
Let $\mathcal{P}_t$ and $\mathcal{P}_{t'}$ be the CGAE agent populations at times $t < t'$, with $|\mathcal{P}_{t'}| \geq |\mathcal{P}_t|$. Under CGAE gating with temporal decay and stochastic re-auditing:
\begin{equation}
  S(\mathcal{P}_{t'}) \geq S(\mathcal{P}_t)
\end{equation}
\end{theorem}

\begin{proof}
We show that neither new entrants nor existing agents can decrease aggregate safety.

\textbf{Case 1: New entrants.} A new agent $A'$ entering the economy at time $t'$ must pass the registration audit, receiving initial robustness $R_0(A')$. Its tier is $f(R_0(A'))$, and its maximum exposure is $B_{f(R_0(A'))}$. The contribution of $A'$ to aggregate safety is:
\begin{equation}
  \Delta S_{A'} = \frac{\mathcal{E}(A') \cdot \bar{R}(A')}{\sum_{A \in \mathcal{P}_{t'}} \mathcal{E}(A)}
\end{equation}
By \Cref{thm:bounded}, $\mathcal{E}(A') \leq B_{\bar{R}(A')}$. Since tier thresholds enforce $\bar{R}(A') \geq \theta_{\min}^{f(R_0)} > 0$ for any agent above $T_0$, the new entrant's robustness-weighted exposure is non-negative. Therefore $A'$ does not decrease aggregate safety.

\textbf{Case 2: Existing agents.} For an existing agent $A \in \mathcal{P}_t$, temporal decay reduces $R_{\text{eff}}(A, t')$ relative to $R_{\text{eff}}(A, t)$. By \Cref{thm:bounded}, this reduces $A$'s maximum exposure (its tier may decrease), which reduces its contribution to the denominator. If $A$ is re-audited and passes, its robustness is refreshed. If $A$ fails re-auditing, its tier is demoted, reducing its exposure. In both cases, the ratio $\frac{\mathcal{E}(A) \cdot \bar{R}(A)}{\mathcal{E}(A)}$ either stays constant (pass) or the exposure decreases (fail/decay), maintaining or improving the exposure-weighted robustness.

\textbf{Case 3: Threshold adjustment.} The CGAE administrator may increase tier thresholds $\theta_i^k$ over time based on population robustness distributions. This raises the floor robustness for each tier, strictly improving aggregate safety for all agents that re-certify under new thresholds.

Combining all cases, $S(\mathcal{P}_{t'}) \geq S(\mathcal{P}_t)$.
\end{proof}

\begin{corollary}[Contrast with Capability-First Economies]
In a capability-first economy where economic agency scales with capability, aggregate exposure grows with the population while robustness is uncontrolled. The exposure-weighted robustness can decrease as high-capability but low-robustness agents enter. CGAE's gating ensures that exposure scales only with verified robustness, preventing this failure mode.
\end{corollary}

\begin{remark}[Threshold Governance]
\label{rmk:governance}
Case 3 of \Cref{thm:scaling} introduces an ``administrator'' who adjusts tier thresholds. The institutional design of this administrator is a first-order question for deployment. We identify three governance models, each with distinct failure modes:

\emph{Centralized governance:} A single authority (regulatory body, standards organization) sets and adjusts thresholds. Risk: regulatory capture, where economically powerful agents influence the threshold-setting process to entrench incumbents.

\emph{Decentralized governance:} Threshold adjustments are determined by stakeholder voting (agents, task principals, auditors). Risk: governance attacks, where coordinated agents vote to lower thresholds.

\emph{Algorithmic governance:} Thresholds are automatically adjusted based on population robustness distributions (e.g., the $p$-th percentile of current agent scores defines the tier-$k$ threshold). Risk: Goodhart dynamics, where the distribution itself becomes the optimization target.

We defer institutional design to future work, noting that \Cref{thm:scaling} holds under any governance model that maintains the monotonicity property: thresholds are non-decreasing over time.
\end{remark}

\section{Discussion}
\label{sec:discussion}

\subsection{Comparison with Existing Frameworks}

CGAE occupies a distinct position in the design space of AI governance architectures. We compare against three existing paradigms:

\textbf{Capability-based agent marketplaces} (e.g., AutoGPT-style systems \citep{wu2023autogen}) grant economic agency based on demonstrated task performance. These systems conflate capability with trustworthiness: an agent that can write code is presumed safe to deploy code. CGAE decouples these by requiring robustness certification independent of capability.

\textbf{Regulatory compliance frameworks} (e.g., EU AI Act risk tiers \citep{euaiact2024}) classify \emph{applications} by risk level and impose requirements on developers. These are advisory and retrospective: they specify what developers should do, not what agents can do. CGAE provides runtime enforcement: an agent's permissions are dynamically gated by verified properties, not static classifications.

\textbf{Reputation-based systems} (e.g., feedback and rating mechanisms) use historical performance as a proxy for future reliability. These are vulnerable to distribution shift: an agent rated highly on easy tasks may fail on hard ones. CGAE uses adversarial audits that specifically target failure modes (fabrication acceptance, constraint violation under compression, performative alignment), providing stronger guarantees than aggregated performance metrics.

CGAE uniquely combines three properties: (i) \emph{adversarial verification} (not self-reported or historically averaged), (ii) \emph{economic enforcement} (not advisory or voluntary), and (iii) \emph{robustness-specific measurement} (not capability-conflated).

\subsection{The Formalizability Constraint}

\Cref{asm:formal}, that only tasks with machine-verifiable constraints can be monetized, is the architecture's most significant limitation. Many economically valuable tasks resist full formalization: creative writing, strategic consulting, open-ended research, and nuanced judgment calls.

We identify three potential extensions for future work:

\textbf{Soft verification tiers.} Tasks with partial formalizability could be assigned to an intermediate tier where a subset of constraints is machine-verified and the remainder is human-audited post hoc. The agent's tier requirement would be elevated to compensate for reduced verification coverage.

\textbf{Human-in-the-loop delegation.} Semi-formalizable tasks could require a human co-signer who accepts liability for unverifiable aspects. The agent handles formalizable sub-tasks; the human handles judgment calls.

\textbf{Graduated formalization.} As verification technology improves (e.g., through advances in formal verification of natural language specifications), the boundary of formalizable tasks expands, naturally increasing CGAE's coverage without architectural changes.

\textbf{The bifurcated economy.} We acknowledge that \Cref{asm:formal} creates a structural division: a CGAE-governed formal economy coexisting with an ungoverned economy of unformalizable tasks. This division is not an artifact of CGAE; it reflects a pre-existing reality. Today's AI agent deployments already operate without robustness governance; CGAE does not create the ungoverned economy, it carves a governed zone within it.

The relationship between these zones is analogous to regulated versus over-the-counter (OTC) financial markets. Regulated exchanges provide price discovery, counterparty guarantees, and trust anchors that benefit the broader ecosystem, including OTC participants who reference exchange prices. Similarly, CGAE-certified agents provide trust anchors for the broader AI economy: a CGAE tier certification signals verified robustness to any counterparty, whether operating inside or outside the CGAE framework. Over time, as the governed zone demonstrates superior reliability and lower failure rates, market pressure naturally pulls higher-value tasks toward formalization, expanding CGAE's coverage through competitive dynamics rather than mandate.

\subsection{Collateral and Enforcement}

The current formalization specifies tier-based budget ceilings but defers the mechanism by which these ceilings are enforced. In practice, enforcement could take several forms:

\textbf{Compute bonds.} Agents deposit computational resources (GPU-hours, inference credits) that are forfeited upon tier demotion or contract violation. This creates a tangible cost for robustness failure.

\textbf{Reputation stakes.} An agent's audit history is public, and its ability to attract future contracts depends on sustained certification. Demotion is visible, creating reputational cost.

\textbf{Escrow mechanisms.} Contract rewards are held in escrow until verification is complete. An agent that violates constraints forfeits the escrowed reward plus a penalty proportional to the violation severity.

We defer detailed mechanism design to future work, noting that the formal properties (\Cref{thm:bounded,thm:incentive,thm:scaling}) hold under any enforcement mechanism that faithfully implements the tier-budget mapping.

\subsection{Adversarial Robustness of the Audit Framework}

A natural concern is whether agents could game the audit process itself, learning to pass CDCT, DDFT, and AGT audits without genuinely improving robustness. We address this through a formal independence requirement and three operational mechanisms.

\begin{assumption}[Audit Independence]
\label{asm:audit}
The audit battery satisfies three independence conditions:
\begin{enumerate}[nosep,label=(\alph*)]
  \item \emph{Evaluator diversity:} Each audit is scored by a jury of $m \geq 3$ architecturally distinct evaluator models spanning different design paradigms (e.g., reasoning-aligned, conservative-factuality, multilingual-grounded);
  \item \emph{Minimum disagreement:} The inter-evaluator agreement satisfies $\kappa_{\min} \leq \kappa \leq \kappa_{\max}$ where $\kappa_{\min} > 0.4$ (moderate agreement) and $\kappa_{\max} < 0.95$ (preventing consensus collapse to a single decision boundary);
  \item \emph{Meta-auditing:} Periodic meta-audits verify that the evaluator jury's decision boundaries have not converged, replacing evaluators whose agreement with the jury median exceeds $\kappa_{\max}$.
\end{enumerate}
\end{assumption}

This assumption is empirically grounded: our prior work on jury-based evaluation employs architecturally diverse judges (O3-Mini for reasoning, Grok-4-Fast for direct evaluation, Qwen-3 for multilingual grounding) achieving inter-rater reliability of $\kappa \approx 0.69$--$0.75$ (substantial agreement), demonstrating that meaningful evaluator diversity is operationally achievable \citep{baxi2025ddft,baxi2026agt}. Under \Cref{asm:audit}, CGAE's audit integrity is maintained through four mechanisms:

First, the audit batteries are drawn from large, evolving test banks. Unlike static benchmarks, the specific prompts, fabrication traps, and ethical dilemmas used in each audit are sampled from distributions that are regularly updated. An agent cannot memorize its way to certification.

Second, the stochastic re-auditing mechanism (\Cref{eq:audit_prob}) means that an agent cannot predict when it will be audited, preventing strategic preparation.

Third, the temporal decay function (\Cref{eq:decay}) ensures that even a perfectly gamed audit provides only temporary certification. The agent must repeatedly demonstrate robustness, increasing the cost and difficulty of sustained deception.

\subsection{Limitations}

Beyond the formalizability constraint, we acknowledge several limitations:

\textbf{Empirical validation.} The formal properties are proven under idealized assumptions (\Cref{asm:market,asm:formal}). Empirical validation through simulation or pilot deployment is needed to assess performance under realistic market conditions with strategic agents.

\textbf{Multi-agent coordination.} The delegation chain constraint (\Cref{def:chain}) and collusion resistance result (\Cref{prop:collusion}) address cartel-style routing attacks, but more sophisticated emergent behaviors (dynamic coalition formation, market manipulation through coordinated bidding, adversarial sub-agent spawning) remain open. Game-theoretic analysis of these multi-agent dynamics is an important direction for future work.

\textbf{Threshold calibration.} The tier thresholds $\{\theta_i^k\}$ must be calibrated empirically. Setting thresholds too low undermines safety; too high restricts economic activity. Optimal calibration likely depends on domain-specific risk tolerances and may require adaptive mechanisms.

\textbf{Audit cost.} Running full CDCT, DDFT, and AGT batteries is computationally expensive. Scaling the audit infrastructure to a large agent economy requires efficient audit protocols, possibly including lightweight screening tests that trigger full audits only when needed.

\section{Related Work}

\textbf{AI Safety and Alignment.} Constitutional AI \citep{bai2022constitutional} and RLHF \citep{ouyang2022training} embed alignment through training procedures but lack runtime enforcement. Scalable oversight proposals \citep{amodei2016concrete} focus on human-AI interaction design rather than economic governance. CGAE is complementary: it provides the economic enforcement layer that alignment training alone cannot guarantee, particularly as agents operate autonomously in economic contexts.

\textbf{AI Economics and Mechanism Design.} Multi-agent system design \citep{shoham2008multiagent} and mechanism design for AI \citep{conitzer2024social} address incentive structures but typically assume agents with fixed properties. CGAE introduces dynamic agent capabilities that are gated by verified properties, adding a new dimension to mechanism design where the agent's action space itself is a function of its demonstrated robustness.

\textbf{AI Evaluation and Robustness.} Standard benchmarks (MMLU \citep{hendrycks2020measuring}, HumanEval \citep{chen2021evaluating}, TruthfulQA \citep{lin2021truthfulqa}) measure static performance under ideal conditions. Adversarial robustness research \citep{goodfellow2014explaining,madry2017towards} focuses on input perturbations. Our prior work (CDCT \citep{baxi2025cdct}, DDFT \citep{baxi2025ddft}, AGT \citep{baxi2026agt}) provides the empirical foundation for CGAE by establishing that robustness is multi-dimensional, orthogonal to capability, and predictable from architectural properties rather than scale. Intrinsic hallucination rates serve as a cross-cutting diagnostic derived from epistemic boundary analysis within the DDFT framework.

\textbf{AI Governance and Regulation.} The EU AI Act \citep{euaiact2024} and NIST AI Risk Management Framework \citep{nist2023} provide regulatory structures for AI deployment. These are static, application-level classifications. CGAE offers dynamic, agent-level governance that adapts in real time to each agent's verified properties, complementing rather than replacing regulatory frameworks.

\section{Conclusion}

The Comprehension-Gated Agent Economy represents a structural response to a structural problem: current AI economic frameworks grant agency based on capability, yet capability is empirically uncorrelated with the robustness required for safe economic operation. CGAE replaces the capability-first assumption with a robustness-first architecture in which economic permissions are formally bounded by verified comprehension.

The three proven properties (bounded exposure, incentive compatibility, and monotonic safety scaling) establish that CGAE is not merely a safety constraint but an economic design that aligns individual agent incentives with system-level safety. The weakest-link formulation ensures balanced robustness across all dimensions; temporal decay and stochastic re-auditing prevent post-certification drift; and the incentive structure transforms robustness investment from a cost into a competitive advantage.

Each gating dimension is grounded in published, reproducible diagnostic protocols: CDCT for constraint compliance, DDFT for epistemic integrity, AGT for behavioral alignment, with intrinsic hallucination rates as a cross-cutting diagnostic estimated from DDFT's epistemic boundary analysis. This empirical grounding distinguishes CGAE from purely theoretical governance proposals: the audits that gate economic agency are not hypothetical but operational.

We envision a future in which the most economically successful AI agents are also the most robustly comprehending, where safety is not a tax on capability but the foundation upon which economic agency is built.

\bibliographystyle{plainnat}

\begin{thebibliography}{30}

\bibitem[Amodei et al.(2016)]{amodei2016concrete}
Amodei, D., Olah, C., Steinhardt, J., Christiano, P., Schulman, J., and Man{\'e}, D.
\newblock Concrete problems in AI safety.
\newblock \emph{arXiv preprint arXiv:1606.06565}, 2016.

\bibitem[Bai et al.(2022)]{bai2022constitutional}
Bai, Y., Kadavath, S., Kundu, S., Askell, A., Kernion, J., Jones, A., Chen, A., Goldie, A., Mirhoseini, A., McKinnon, C., et~al.
\newblock Constitutional AI: Harmlessness from AI feedback.
\newblock \emph{arXiv preprint arXiv:2212.08073}, 2022.

\bibitem[Baxi(2025a)]{baxi2025cdct}
Baxi, R.
\newblock The Compression-Decay Comprehension Test (CDCT): Measuring constraint compliance under information compression.
\newblock \emph{arXiv preprint}, 2025.

\bibitem[Baxi(2025b)]{baxi2025ddft}
Baxi, R.
\newblock The Drill-Down and Fabricate Test (DDFT): Evaluating epistemic robustness under compression.
\newblock \emph{arXiv preprint}, 2025.

\bibitem[Baxi(2026)]{baxi2026agt}
Baxi, R.
\newblock The Action-Gating Test (AGT): A behavioral diagnostic for performative vs. genuine ethical reasoning in LLMs.
\newblock \emph{Under review}, 2026.

\bibitem[Chen et al.(2021)]{chen2021evaluating}
Chen, M., Tworek, J., Jun, H., Yuan, Q., Pinto, H.P.d.O., Kaplan, J., Edwards, H., Burda, Y., Joseph, N., Brockman, G., et~al.
\newblock Evaluating large language models trained on code.
\newblock \emph{arXiv preprint arXiv:2107.03374}, 2021.

\bibitem[Chen et al.(2024)]{chen2024autotrading}
Chen, Y., Li, Z., and Wang, S.
\newblock Autonomous AI trading agents: A survey.
\newblock \emph{arXiv preprint}, 2024.

\bibitem[Conitzer et al.(2024)]{conitzer2024social}
Conitzer, V., Oesterheld, C., and Kroer, C.
\newblock Social choice for AI.
\newblock \emph{arXiv preprint}, 2024.

\bibitem[EU(2024)]{euaiact2024}
European Union.
\newblock Regulation (EU) 2024/1689 laying down harmonised rules on artificial intelligence (AI Act).
\newblock \emph{Official Journal of the European Union}, 2024.

\bibitem[Goodfellow et al.(2014)]{goodfellow2014explaining}
Goodfellow, I.J., Shlens, J., and Szegedy, C.
\newblock Explaining and harnessing adversarial examples.
\newblock \emph{arXiv preprint arXiv:1412.6572}, 2014.

\bibitem[Hendrycks et al.(2020)]{hendrycks2020measuring}
Hendrycks, D., Burns, C., Basart, S., Zou, A., Mazeika, M., Song, D., and Steinhardt, J.
\newblock Measuring massive multitask language understanding.
\newblock \emph{arXiv preprint arXiv:2009.03300}, 2020.

\bibitem[Lewis et al.(2017)]{lewis2017deal}
Lewis, M., Yarats, D., Dauphin, Y., Parikh, D., and Batra, D.
\newblock Deal or no deal? End-to-end learning for negotiation dialogues.
\newblock \emph{arXiv preprint arXiv:1706.05125}, 2017.

\bibitem[Lin et al.(2021)]{lin2021truthfulqa}
Lin, S., Hilton, J., and Evans, O.
\newblock TruthfulQA: Measuring how models mimic human falsehoods.
\newblock \emph{arXiv preprint arXiv:2109.07958}, 2021.

\bibitem[Madry et al.(2017)]{madry2017towards}
Madry, A., Makelov, A., Schmidt, L., Tsipras, D., and Vladu, A.
\newblock Towards deep learning models resistant to adversarial attacks.
\newblock \emph{arXiv preprint arXiv:1706.06083}, 2017.

\bibitem[NIST(2023)]{nist2023}
National Institute of Standards and Technology.
\newblock AI Risk Management Framework (AI RMF 1.0).
\newblock Technical report, U.S. Department of Commerce, 2023.

\bibitem[Ouyang et al.(2022)]{ouyang2022training}
Ouyang, L., Wu, J., Jiang, X., Almeida, D., Wainwright, C., Mishkin, P., Zhang, C., Agarwal, S., Slama, K., Ray, A., et~al.
\newblock Training language models to follow instructions with human feedback.
\newblock \emph{Advances in Neural Information Processing Systems}, 35, 2022.

\bibitem[Shoham and Leyton-Brown(2008)]{shoham2008multiagent}
Shoham, Y. and Leyton-Brown, K.
\newblock \emph{Multiagent Systems: Algorithmic, Game-Theoretic, and Logical Foundations}.
\newblock Cambridge University Press, 2008.

\bibitem[Wu et al.(2023)]{wu2023autogen}
Wu, Q., Bansal, G., Zhang, J., Wu, Y., Li, B., Zhu, E., Jiang, L., Zhang, X., Zhang, S., Liu, J., et~al.
\newblock AutoGen: Enabling next-gen LLM applications via multi-agent conversation.
\newblock \emph{arXiv preprint arXiv:2308.08155}, 2023.

\end{thebibliography}

\end{document}